\documentclass[tablecaption=bottom,wcp]{jmlr}

\usepackage{blindtext}

\usepackage[american]{babel}
\usepackage{microtype}
\usepackage{graphicx}
\usepackage{booktabs}

\usepackage{amsmath}
\usepackage{bm}
\usepackage{amssymb}
\usepackage{mathtools}
\usepackage{multicol}
\usepackage[capitalize,noabbrev]{cleveref}
\usepackage{lastpage}
\usepackage{tikz}
\usepackage{changepage}
\usepackage{soul}

\addto\captionsamerican{
  
}

\DeclareMathOperator*{\argmin}{arg\,min}
\DeclareMathOperator*{\argmax}{arg\,max}

\usetikzlibrary{bayesnet}

\newcommand{\Y}{\mathbf{Y}}
\newcommand{\y}{\mathbf{y}}
\newcommand{\M}{\mathbf{M}}
\newcommand{\X}{\mathbf{X}}
\newcommand{\A}{\mathbf{A}}
\renewcommand{\L}{\mathbf{L}}
\newcommand{\I}{\mathbf{I}}
\renewcommand{\S}{\mathbf{S}}
\newcommand{\T}{\mathbf{T}}
\newcommand{\C}{\mathbf{C}}
\newcommand{\U}{\mathbf{U}}
\renewcommand{\C}{\mathbf{C}}
\newcommand{\blu}{\textcolor{blue}}
\newcommand{\red}{\textcolor{red}}

\newcommand{\q}{l} 

\jmlrproceedings{Preprint}{A preprint.}

\title[Dimensionality Reduction: A Unifying Probabilistic Perspective]{Dimensionality Reduction as Probabilistic Inference}

\author{
   \Name{Aditya Ravuri},
   \Name{Francisco Vargas},
   \Name{Vidhi Lalchand} \and
   \Name{Neil D. Lawrence} \\
   \addr University of Cambridge
}

\begin{document}

\setlength{\parindent}{0pt} 
\maketitle

\begin{abstract}

Dimensionality reduction (DR) algorithms compress high-dimensional data into a lower dimensional representation while preserving important features of the data. DR is a critical step in many analysis pipelines as it enables visualisation, noise reduction and efficient downstream processing of the data. In this work, we introduce the \textit{ProbDR} variational framework, which interprets a wide range of classical DR algorithms as probabilistic inference algorithms in this framework. ProbDR encompasses PCA, CMDS, LLE, LE, MVU, diffusion maps, kPCA, Isomap, (t-)SNE, and UMAP. In our framework, a low-dimensional latent variable is used to construct a covariance, precision, or a graph Laplacian matrix, which can be used as part of a generative model for the data. Inference is done by optimizing an evidence lower bound. We demonstrate the internal consistency of our framework and show that it enables the use of probabilistic programming languages (PPLs) for DR. Additionally, we illustrate that the framework facilitates reasoning about unseen data and argue that our generative models approximate Gaussian processes (GPs) on manifolds. By providing a unified view of DR, our framework facilitates communication, reasoning about uncertainties, model composition, and extensions, particularly when domain knowledge is present.

\end{abstract}


\section{Introduction}

Many experimental data pipelines, for example, in single-cell biology, generate noisy, high-dimensional data that is hypothesised to lie near a low-dimensional manifold.  Dimensionality reduction algorithms have been used for such problems to find low-dim. embeddings of the data and enable efficient downstream processing. However, to better encode important characteristics of the high-dimensional data, quantify and reduce noise, and remove confounders, a probabilistic approach is needed, especially to encode context specific information.\\

The key motivation for this work is that probabilistic models and interpretations enable composability of assumptions, model extension, and aid communication through explicit definition of priors, model and inference algorithm \citep{zoubin_pml,bda3}. In single-cell data analysis, inductive biases have been encoded via priors in GPLVMs, for example pseudotime with von Mises priors and periodic covariances \citep{grandprix, vidhi-emma-pseudotime}. In the context of DR, probabilistic interpretations have offered ways to deal with missing data and formulate probabilistic mixtures \citep{ppca}. \\

A number of algorithms, such as \textbf{PCA, FA, GMMs, NMF, LDA, ICA} \citep{pml-ii} are known to have probabilistic interpretations, wherein the generative model for $n$ independent high ($d$-)dimensional data points $\Y \equiv \begin{bmatrix} \Y_{1:} & ... & \Y_{n:} \end{bmatrix}^T \in \mathbb R^{n \times d}$ is,
\begin{align}
 \X_{i:} &\sim p(.), \nonumber \\
 \Y_{i:} | \X_{i:} &\sim \text{ExponentialFamily}(f(\X_{i:}))
\label{eqn:gp}
\end{align}
where $\X \in \mathbb R^{n \times \q}$ is a matrix-valued random variable of corresponding low ($\q$-)dimensional embeddings/latent variables. The inference process is typically full-form (i.e. unamortised) as inference occurs for the full matrix $\X$. Vanilla \textbf{GPLVMs} \citep{gplvm} and \textbf{VAEs} \citep{vae} were also designed with such generative models, with the map between the latents and the data distribution's parameters $f$ being described using a GP and a neural network respectively, rather than a linear function. Our work provides a novel probabilistic perspective unifying a wider class DR algorithms not known to have interpretations as inferences of probabilistic models, to the best of our knowledge. We list our contributions below, summarize them in \cref{fig:probdr_summary} and motivate the work below.
\begin{itemize}
 \item We introduce the ProbDR model framework, and show how \textbf{SNE, t-SNE and UMAP} correspond to different instances of the inference algorithm under our framework.
 \item We show that many DR methods estimating an embedding as eigenvectors of a PSD matrix perform a two-step process (referred to henceforth as ``2-step MAP'') where,
 \begin{enumerate}
     \item one estimates a PSD moment matrix $\hat{\M}$ (e.g. representing a covariance $\hat{\S}$ or precision matrix $\hat{\mathbf{\Gamma}}$) using high dimensional data $\Y$,
     \item then estimates the embedding via maximum-a-posteriori (MAP) inference in a probabilistic model involving a Wishart distribution, resulting in the eigencomps.
 \end{enumerate}
 \item We show that 2-step MAP is equivalent to inference in ProbDR, and that \textbf{CMDS, LLE, LE, MVU, Isomap, diffusion maps \& kPCA} have ProbDR interpretations.
    \item We show examples reproducing embeddings of canonical implementations using PPLs, enabled by ProbDR, and show that ProbDR also enables reasoning about unseen data.
\end{itemize}

\begin{figure}
\centering
\vspace{-1cm}
\begin{tikzpicture}
\node[obs, draw=white](vi) {ProbDR};
\node[obs, right=of vi, draw=white](mle) {2-step MAP};
\node[latent, right=of mle, draw=white](s) {eigencomp. DR};
\node[latent, left=of vi, draw=white] (umap) {(t-)SNE, UMAP};

\edge{mle}{vi};
\edge{umap}{vi};
\edge{s}{mle};

\plate{ch2} {(umap)(vi)} {\textcolor{cyan}{\cref{section:probdr}, \cref{tbl:tsne-umap}}};
\plate{ch2} {(mle)(vi)} {\textcolor{orange}{\cref{sec:two-step-mle}}};
\plate{ch2} {(mle)(s)} {\textcolor{magenta}{\cref{tbl:two-step}}};

\end{tikzpicture}
\vspace{-0.5cm}
\caption{
  A figure summarizing our work. \textbf{Left:} UMAP and (t-)SNE have direct ProbDR interpretations. \textbf{Right:} the two-step MAP process, which describes DR methods that rely on eigencomponents, is equivalent to inference in ProbDR.
}
\label{fig:probdr_summary}
\end{figure}
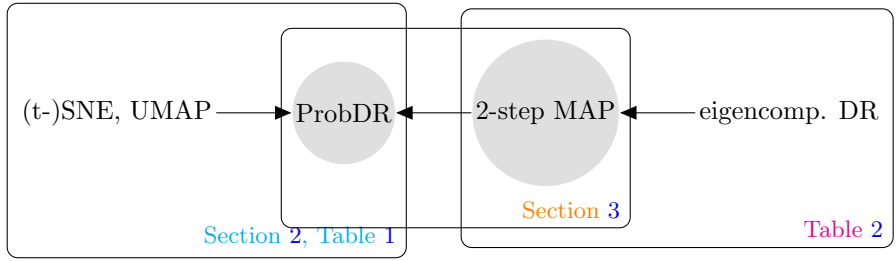

\section{The ProbDR Model Framework \& Inference}
\label{section:probdr}

ProbDR is a variational framework, illustrated in \cref{fig:probdr}, in which low dimensional latents $\X$ describe a moment or summary statistic of the data $\M$ (e.g. a covariance), using which a generative model on the data $\Y$ is constructed. The moment $\M$ has a variational distribution associated with it, that uses the data $\Y$ (as in VAEs and backconstrained/variational GPLVMs (\cite{bui-turner} based on \cite{gplvm-backconstraints})). \\

Inference in the framework is done by maximising a lower bound on the evidence (and the likelihood), the ELBO \citep{vi-intro, blei-vi}, w.r.t. $\X$ and model parameters,
\begin{align}
\label{eqn:probdr-obj}
\argmax_{\X, \theta} \mathbb{E}_{q(\M|\Y)}[\log p_{\theta}(\Y|\M)] - \text{KL}(q(\M|\Y)||p(\M | \X)).
\end{align}

A derivation is given in \cref{app:tsne-umap}. In our framework, the variational distribution $q$ does not have any parameters that are optimised, much like the case of denoising diffusion models \citep{diff-models}, and unlike traditional variational inference \citep{blei-vi}.

The objective above has two terms. The second term (the KL divergence) corresponds to the objective/cost function that is minimised in each of the respective DR algorithms. The first term corresponds to the generative model placed using the moment $\M$ on data $\Y$ and has no dependence on latents $\X$. Therefore, the generative model is a ``free'' addition, as its presence adds a constant to the objective w.r.t. the latents.

\subsection*{(t-)SNE \& UMAP}

(t-)SNE \& UMAP correspond to inference in the ProbDR framework, with different distributions placed on a random adjacency matrix $\M \equiv \A' \in \{0, 1\}^{n \times n}$, representing a data-data similarity matrix. (t-)SNE and UMAP define probabilities of data similarity $v_{ij}$ that depend on distances between the high-dim. datapoints $\Y_{i:}$ and $\Y_{j:}$, and $w_{ij}$ that depend on the distances between the low-dim. latents $\X_{i:}$ and $\X_{j:}$.

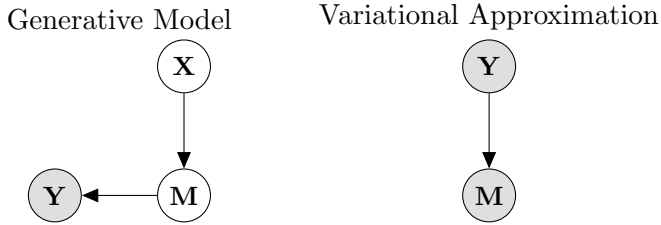
\begin{figure}
\centering
\begin{tikzpicture}
\node[latent] (A) {$\M$};
\node[latent, above=of A] (X) {$\X$};
\node[obs, left=of A] (Y) {$\Y$};

\edge{X}{A};
\edge{A}{Y};

\node[above] at (current bounding box.north) {Generative Model \newline};
\end{tikzpicture}
\qquad
\begin{tikzpicture}
\node[obs] (A) {$\M$};
\node[obs, above=of A] (Y) {$\Y$};

\edge{Y}{A}; 

\node[above] at (current bounding box.north) {Variational Approximation};
\end{tikzpicture}
\caption{A simplified graphical model that summarizes the ProbDR class of models. ELBOs corresponding to these models give rise to (t-)SNE, UMAP and other objectives.}
\label{fig:probdr}
\end{figure}

\begin{theorem}
    \label{thm:tsne-umap}
    (t-)SNE and UMAP objectives are recovered as the KL div. in \cref{eqn:probdr-obj} when model \& variational distributions on an adjacency matrix $\A'$ are set as in \cref{tbl:tsne-umap}.
\end{theorem}

\begin{table}[h]
    \centering
    \begin{tabular}{l|cc|c}
    algo & $q(\A'|\Y)$ & $p(\A' | \X)$ & $\text{KL}(q||p)$ \\
    \hline
    UMAP & $\prod_{i \neq j}^n \text{Bernoulli}(\A'_{ij} | v_{ij}^U(\Y))$ & $\prod_{i \neq j}^n \text{Bernoulli}(\A'_{ij} | w_{ij}^U(\X_{i:, j:}))$ & $\mathcal{C}_{\text{UMAP}}$ \\
    SNE & $\prod_i^n \text{Categorical}(\A'_{i:} | v_{i:}^S(\Y))$ & $\prod_i^n \text{Categorical}(\A'_{i:} | w_{ij}^S(\X_{i:, j:}))$ & $\mathcal{C}_{\text{SNE}}$ \\
    t-SNE & $\text{Categorical}(\text{vec}(\A') | v_{::}^t(\Y))$ & $\text{Categorical}(\text{vec}(\A') | w_{::}^t(\X))$ & $\mathcal{C}_{\text{t-SNE}}$ \\
    \end{tabular}
    \caption{ProbDR assumptions that result in (t-)SNE \& UMAP objectives.}
    \label{tbl:tsne-umap}
\end{table}

Throughout this work, we assume an improper uniform prior on $\X$, $p(\X) \propto 1$. See \cref{app:tsne-umap} proofs, further detail and a discussion on why we flip notation w.r.t. the (t-)SNE papers (i.e. our objective appears as $\text{KL}(q\|p)$ instead of $\text{KL}(p\|q)$) although the computation is identical.

\subsubsection*{Generative Models for (t-)SNE \& UMAP}

Generative models in the ProbDR framework allow for inferences to be done at the data level (e.g. reconstructions, out of data predictions) using latent variables obtained through the various DR algorithms.
Any generative model $p(\Y|\A')$ that depends only on the adjacency matrix is a valid generative model for the (t-)SNE and UMAP cases, however, a natural choice is a Matérn-$\nu$ graph Gaussian process \citep{ggp}.

If $\A$ is a symmetric adjacency matrix (calculated as $\A_{ij} = \A'_{ij} \lor \A'_{ji}$), then a suitable graph Laplacian $\L$ can be derived (e.g. the ordinary $\L = \mathbf{D} - \A$, or the normalized $\L = \I - {\mathbf{D}^{\dagger}}^{1/2} \A {\mathbf{D}^{\dagger}}^{1/2}$) and a generative model can be specified as,
$$ \forall i: \Y_{:i} | \L \sim \mathcal N \left(\mathbf{0}, \begin{cases}
    [\L + \beta \I]^{-1} & \text{Matérn-1 case} \\
    \exp[-t\L] & \text{Matérn-}\infty \text{ case}
\end{cases} \right). $$
Note that the Matérn-1 case is the the Gaussian Markov random field covariance of \cite{neil-gmrf}. The normalised Laplacian is more useful in practice as graph statistics (e.g. degrees) implied by the variational and model distributions on $\A'$ are quite different (esp. in the UMAP case).
Due to the additional dependence on $\L$ as opposed to the GPs in \cref{eqn:gp},
\begin{enumerate}
    \item these generative models lack of marginal consistency - i.e. $\Y_{i:}$ indexed by $\X_{i:}$ can't be described by a GP as every new data point changes the full covariance of the data;
    \item the generative model has non-uniform marginal variances.
\end{enumerate}
See \cref{app:tsne-umap} for discussions on adjacency matrices, marginal variances and properties of $\L$ that make it a suitable precision. \cref{app:expts} shows that, despite these limitations, prior samples using graph GPs indexed using $\X$ resemble samples from traditional GPs.

\section{Two-step MAP \& the Wishart model class}
\label{sec:two-step-mle}

Next, we focus on the 2-step MAP class of algorithms and show equivalence to ProbDR. Many DR algorithms estimate an embedding as a two step process (exemplified in \cref{tbl:two-step}),
\begin{enumerate}
    \item Estimate a PSD matrix $\hat{\M}$, which we interpret as a moment (a covariance $\hat{\S}$ or precision $\hat{\mathbf{\Gamma}}$). This can be a function of the data, e.g. PCA, where $\hat{\M}(\Y) \equiv \hat{\S}(\Y) = \Y\Y^T/d$ or as a result of a likelihood maximisation, i.e. $\hat{\M}(\Y) = \argmax_{\M} \mathcal{L}(\Y; \M)$ as in LLE.
    \item Set the embedding $\X$ to $\q$ scaled eigenvectors of $\hat{\M}$ corresponding to the largest or lowest eigenvalues (referred to as major \& minor eigenvectors respectively).
\end{enumerate}
To draw a connection to ProbDR, firstly, we show that step 2 is MAP estimation for $\X$,
\begin{align}
    \vspace{-0.5cm}
    \hat{\X} &= \argmax_{\X} \log p(\X | \hat{\M} * d) \overset{\text{bayes}}{=} \argmax_{\X} \log p(\hat{\M} * d| g(\X)),
    \vspace{-0.5cm}
\label{eqn:wish-mle}
\end{align}
setting the model for $p(\hat{\M} * d | \X)$ to be a Wishart distribution as per \cref{thm:two-step-mle}, $p(\X) \propto 1$ and where $g(.)$ computes the mean parameter of the Wishart.
\begin{theorem}[Step 2 is MAP estimation] The MAP estimate of $\X$, with an improper uniform prior over $\X$, given the models below occurs at the $\q$ principal/major and minor scaled eigenvectors of $\hat{\S}$ \& $\hat{\mathbf{\Gamma}}$ respectively,
\begin{align}
    \label{eqn:two-step-mle}
    \hat{\mathbf{\S}} * d | \mathbf{X} \sim \mathcal{W}\left(\X \X^T + \sigma^2 \I_n, d\right) \quad &\Rightarrow \quad \hat{\X}_{\text{MAP}} = \mathbf{U}_{\q \text{ maj}} (\mathbf{\Lambda}_{\q \text{ maj}} - \hat \sigma^2 \I_{\q})^{1/2} \mathbf{R}^T \\
    \hat{\mathbf{\Gamma}} * d | \mathbf{X} \sim \mathcal{W}\left((\X \X^T + \beta \I_n)^{-1}, d\right) \quad &\Rightarrow \quad \hat{\X}_{\text{MAP}} = \mathbf{U}_{\q \text{ min}} (\mathbf{\Lambda}^{-1}_{\q \text{ min}} - \hat \beta \I_{\q})^{1/2} \mathbf{R}^T \nonumber
\end{align}
where $\U_\q, \mathbf{\Lambda}_\q$ are matrices of $\q-$eigenvectors and corresponding eigenvalues and $\mathbf{R}$ is an arbitrary rotation matrix. Proved in \cref{app:two-step}. $\hat{\S} = \Y\Y^T/d$ recovers PCA; see \cref{thm:pca}.
\label{thm:two-step-mle}
\end{theorem}

Secondly, to establish a connection to ProbDR, we state (and prove in \ref{app:probdr-map-equiv}) that the 2-step process of estimating $\hat{\M}$ and performing MAP as in \cref{eqn:wish-mle} is equivalent to ProbDR.
\begin{theorem}
    Finding $\argmin_{\X} \text{KL}( \; q(\M | \hat \M(\Y)) \; \| \; p(\M | \X) \; )$, i.e. the ProbDR KL div. of \cref{eqn:probdr-obj}, is equivalent to 2-step MAP assuming,
    $$ p(\M | \X) = \mathcal{W}(\M | g(\X), d) \text{ and } q(\M|\hat \M) = \mathcal{W}(\M| \hat{\M}(\Y), d). $$
\end{theorem}
It's interesting to note that discarding the variational assumption and marginalising the covariance/precision of \cref{thm:two-step-mle} using certain multivariate normal generative models leads to a GP with a linear kernel (as in PCA; \cref{thm:pca_marginal}), showing consistency of ProbDR.

\begin{table}
    \centering
    \vspace{-0.9cm}
    \begin{adjustwidth}{-1cm}{}
    \begin{tabular}{p{17cm}}
    \toprule
    \textbf{CMDS, Isomap, kernel PCA \& MVU}: \\
    \textbf{Step 1}: Each algorithm first calculates a matrix $\mathbf{K}$. CMDS and Isomap set $\mathbf{K} = -0.5 * \text{(a distance matrix)}$. This is computed outright using a metric in the case of CMDS. In the case of Isomap, nearest neighbours are identified for every datapoint, and then the distance matrix is set to a matrix of shortest distances on the neighbour graph. In the case of kPCA, it is computed using a kernel evaluated on data point pairs, $\mathbf{K}_{ij} = k(\Y_{i:}, \Y_{j:})$. MVU estimates $\mathbf{K}$ by maximising $\text{tr}(\mathbf{K})$ under PSD, centering and local isometry constraints. Then, in all methods, $\mathbf{K}$ is centered using a centering matrix $\mathbf{H}$, \vspace{-0.3cm}
    $$\hat{\S}(\Y) = \mathbf{H} \mathbf{K} \mathbf{H}.\vspace{-0.3cm}$$
    The centered matrix $\hat{\S}$ has an interpretation as a similarity matrix. Although in the latter cases $\hat{\S}$ is PSD, it isn't generally (e.g. with non-Euclidean distances in CMDS). In CMDS and for the purposes of ProbDR, an approximated PSD matrix $\hat{\S}^+$ is used. This is found by obtaining the eigendecomposition of $\hat{\S}$ and setting non positive eigenvalues to zero. References: \cite{neil-gmrf, isomap, mds, mvu}. \\
    \textbf{Step 2}: The embedding is found using major scaled eigencomps of $\hat{\S}$, which are usually also the major scaled eigencomps of $\hat{\S}^+$, as generally, only the minor eigenvectors of $\hat{\S}$ are removed.
    \\
    \hline
    \textbf{Laplacian Eigenmaps \& Spectral Embeddings}: \\
    \textbf{Step 1}: LE constructs a normalized, weighted graph Laplacian, encoding data similarity. E.g.,\vspace{-0.25cm} $$ \hat{\mathbf{\Gamma}}(\Y)_{ij} = \Tilde{\L}_{ij} \text{ with } \A_{ij} =  \mathcal I(\|\Y_{i:} - \Y_{j:}\| < \epsilon)\vspace{-0.3cm}$$
    and a graph Laplacian, by assumption, is present for the embedding stage of spectral clustering. \\
    \textbf{Step 2}: The embedding is set to $\q$-minor eigenvectors of $\Tilde{\L}$ but after discarding the first minor (the constant) eigenvector. It's known that setting,\vspace{-0.25cm}
    $$\hat{\S}(\Y) = \mathbf{H} (\Tilde{\L} + \gamma \I)^{-1} \mathbf{H},\vspace{-0.3cm}$$
    and obtaining $\q$ major eigenvectors results in the disposal of the constant eigenvector. Hence, either case ($\hat{\mathbf{\Gamma}}$ or $\hat{\S}$) is a valid PSD matrix in ProbDR, although the ProbDR solution will differ up to scaling (and potentially the inclusion of the constant eigenvector if $\hat{\mathbf{\Gamma}}$ is chosen). References: \cite{neil-gmrf, lap-eigenmaps, graph-tutorial}. \\
    \hline
    \textbf{Locally Linear Embedding}: \\
    \textbf{Step 1}: The (generalised) LLE algorithm can be interpreted to first perform inference on ``reconstruction weights'' $\mathbf{W}$ via pseudolikelihood optimisation given the model,\vspace{-0.25cm}
    $$ \forall i: \Y_{i:} | \Y_{-i} \sim \mathcal N \Big( -{\mathbf{W}^{-1}_{ii}} \sum_{j \in \mathcal N(i)} \mathbf{W}_{ji}\Y_{j:}, {\mathbf{W}_{ii}^{-2}}\Big); \text{ with } \hat{\mathbf{\Gamma}}(\Y) = \L = \mathbf{W} \mathbf{W}^T, \vspace{-0.25cm}$$
    $\mathbf{W}_{ii} = -\sum_{j\in \mathcal N(i)} \mathbf{W}_{ji} = 1$ and $\forall j \not \in \mathcal N(i): \mathbf{W}_{ji} = 0$. Reference: \cite{neil-gmrf}. \\
    \textbf{Step 2}: This is exactly as in the case of Laplacian Eigenmaps. \\
    \hline
    \textbf{Diffusion maps}: \\
    \textbf{Step 1}: The diffusion maps algorithm computes a ``transition matrix'' $\mathbf{P}$ (a kernel matrix evaluated at $\Y$ and normalized). Under specific choices of normalization, the matrix computed estimates a heat kernel (a Matérn-$\infty$ graph covariance). Other methodologies use similar matrices (e.g. SR matrices, whose eigenvectors resemble spatial eigenfunctions). $\hat{\S}$ may represent such matrices if they're PSD (approximately; after centering if needed). Ref.: \cite{sr-eigenfuns, diff-maps}. \\
    \textbf{Step 2}: Major scaled eigencomps are obtained as the embedding.\\ \bottomrule
    \end{tabular}
    \end{adjustwidth}
    \caption{Algorithms that can be interpreted as 2-step MAP processes. Step 1 shows computation of $\hat{\S}$ or $\mathbf{\hat{\Gamma}}$ for \cref{eqn:two-step-mle}. Step 2 highlights that all methods eigendecompose the appropriate matrix, which correspond to MAP inference (as in \cref{thm:two-step-mle}, up to scaling); this table also highlights important nuances. Scaled eigencomps refer to $\q$ major or minor eigenvalue scaled eigenvectors of $\hat{\S}$ or $\mathbf{\hat{\Gamma}}$.}
    \label{tbl:two-step}
\end{table}

\begin{figure}
    \centering
    \includegraphics[width=0.75\columnwidth]{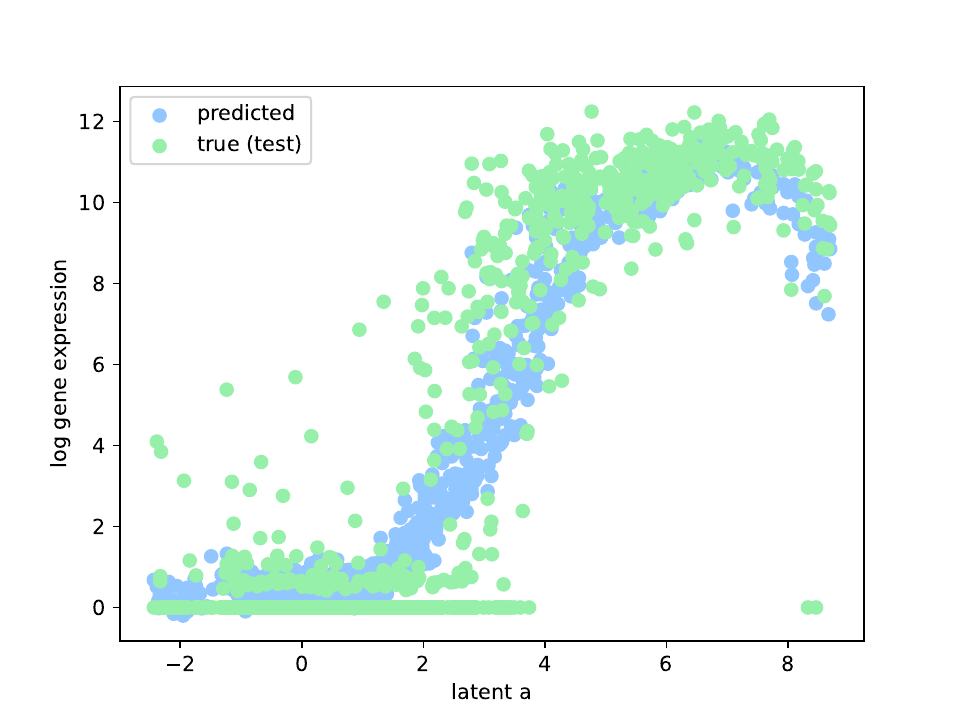}
    \caption{\small The figure shows predictions (blue) for unseen gene expression data given an embedding, generated with a graph GP, fit using the ProbDR-UMAP framework. The fit achieves a better predictive test RMSE than a VAE.}
    \label{fig:expt_b}
\end{figure}

\begin{figure}
    \centering
    \includegraphics[width=0.75\columnwidth]{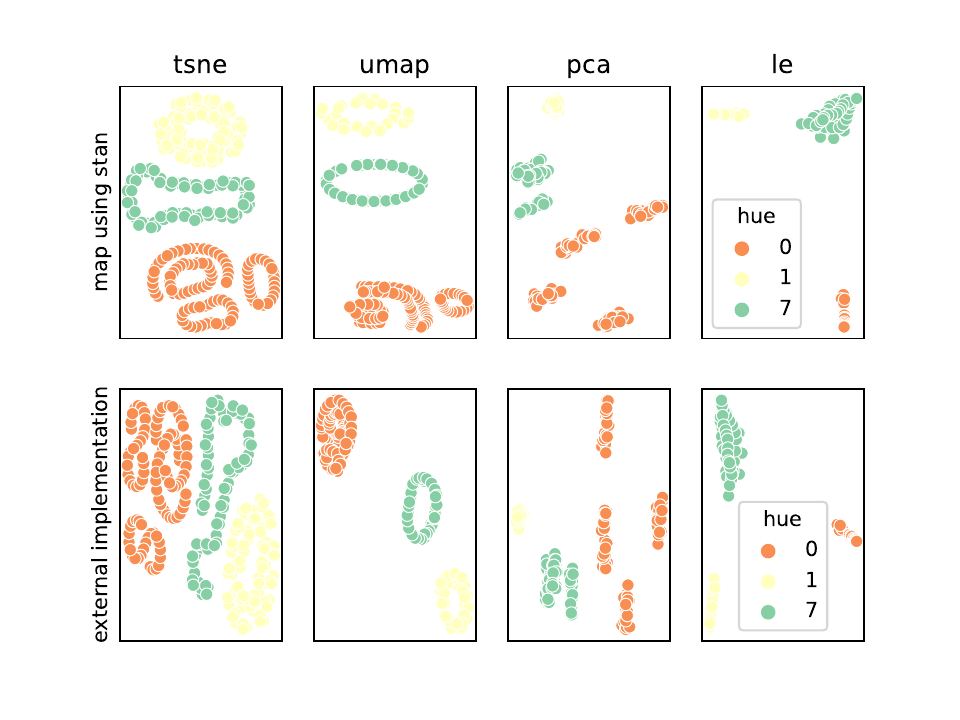}
    \caption{\small The figure shows embeddings of a few rotated MNIST figures recovered through automated MAP estimation using the PPL Stan with ProbDR assumptions, compared with popular community implementations.}
    \label{fig:expt_a}
\end{figure}

\section{Conclusion}

We introduce the ProbDR framework that provides a unified perspective on a large number of DR algorithms. As an immediate consequence, PPLs can be used to perform DR via automated inference (e.g. \cref{fig:expt_a}), and our framework allows these methods to be used with generative models (e.g. \cref{fig:expt_b}; further detail is provided in \cref{app:expts}). We show that our framework is internally consistent (see \cref{app:sub:consistency}), and that marginalising the intermediary moment and discarding the variational constraint leads to GP behaviour in the generative models (\cref{thm:pca_marginal}). Future work will aim to study the characteristics of constraints set by the various DR algorithms' corresponding variational approximations. We will also explore whether these constraints can be used to guide kernel choice for defining GPs on manifolds \citep{manifold-gp}, for instance hyperbolic kernels representing tree structures \citep{hyperbolic} and hyperspherical kernels representing cyclicality.



\newpage
\vskip 0.2in
\bibliography{refs}

\newpage
\appendix

\tableofcontents

\newpage

\section{Proof of (t-)SNE and UMAP results and remarks}
\label{app:tsne-umap}

\subsection{Derivation of the main objective}

We first show that a derivation of the ELBO from first principles in the ProbDR framework, and then show that the individual algorithms ((t-)SNE and UMAP) minimize the KL divergence found in our ELBO. The objective, an evidence lower bound, is derived as follows,
\begin{align*}
\text{KL}(q(\M | \Y)||p_{\theta}(\M | \X, \Y)) &= \mathbb{E}_{q(\M|\Y)} \left[ \dfrac{\log q(\M|\Y)}{\log p_{\theta}(\M|\X, \Y)} \right] \\
&= \mathbb{E}_{q(\M|\Y)} [ \log q(\M|\Y) ] - \mathbb{E}_{q(\M|\Y)} [ \log p_{\theta}(\Y|\M) p(\M|\X) ] + \log p(\Y) \\
&=  \mathbb{E}_{q(\M|\Y)} \left[ \dfrac{\log q(\M|\Y)}{\log p(\M|\X)} \right] - \mathbb{E}_{q(\M|\Y)}[\log p_{\theta}(\Y|\M)] + \log p(\Y) \\
&=  \text{KL}(q(\M|\Y) \| p(\M|\X)) - \mathbb{E}_{q(\M|\Y)}[\log p_{\theta}(\Y|\M)] + \log p(\Y) \\
&= \log p(\Y) - \text{ELBO}(\X, \theta).
\end{align*}
As $\log p(\Y)$ is constant,
$$\argmin_{\theta, \X} \text{KL}(q(\M | \Y)||p_{\theta}(\M |\X, \Y)) = \argmax_{\theta, \X} \text{ELBO}(\X, \theta).$$
In the derivation above, we assume an improper uniform prior over $\X$, i.e. $p(\X) \propto 1$. Our objective may also be interpreted as a regularised Bayesian inference \citep{regbayes} objective.

Optimising the ELBO w.r.t. $\X$ leads to the minimisation problem becoming,
\begin{equation}
    \label{eqn:prob-dr-kl}
    \argmax_{\theta, \X} \text{ELBO}(\X, \theta) = \argmin_{\X} \text{KL}(q(\M | \Y)||p(\M |\X)),
\end{equation}
as the data fit term of the ELBO (the first term in \cref{eqn:probdr-obj}) is independent of $\X$ and the $\text{KL}$ above is independent of $\theta$. Below, we show how this KL divergence of \cref{eqn:prob-dr-kl} arises in (t-)SNE and UMAP.

\subsection{SNE}

The stochastic neighbour embedding (SNE) algorithm was introduced by \cite{sne} as an approach for dimensionality reduction. The approach was to minimise a Kullback Leibler (KL) divergence between a set of probabilities $v^{S}_{ij}$ (corresponding to two data points $i$ and $j$ being neighbours) generated by a discrete distribution in a data space $\Y$ and a discrete distribution with probabilities $w_{ij}^{S}$ generated by using a lower dimensional latent embedding $\X$. These probabilities are defined as,
\begin{align*}
    v_{ij}^{S} &= \frac{\exp(-\|\Y_{i:} - \Y_{j:}\|^2/\sigma_i^2)}{\sum_{k\neq i} \exp(-\|\Y_{i:} - \Y_{k:}\|^2/\sigma_i^2)}, \\
    w_{ij}^{S} &= \frac{\exp(-\|\X_{i:} - \X_{j:}\|^2)}{\sum_{k\neq i} \exp(-\|\X_{i:} - \X_{k:}\|^2)},
\end{align*}
where $\mathbf{M}_{i:}$ and $\mathbf{M}_{:j}$ denote row $i$ and column $j$ of a matrix $\mathbf{M}$  respectively and $\sigma_i$ denotes a hyperparameter. Probabilities $w_{ij}^{S}$ are made close to probabilities $v_{ij}^{S}$ by minimizing the objective below with respect to $\X$,
\begin{align*}
    \mathcal C_{SNE} = \sum_{i} \sum_{j\neq i} v_{ij}^{S} \log  \frac{v_{ij}^{S}}{w_{ij}^{S}}.
\end{align*}
The idea is that if probabilities defined in latent space are similar in terms of the KL divergence to probabilities defined in data space, then the latent dimensions of $\X$ are capturing some salient aspect of the data $\Y$. In all three algorithms, probabilities of association relating to the same point, $v_{ii} \text{ and } w_{ii}$, are set to zero. \\

\begin{proof} of \cref{thm:tsne-umap}, SNE case. Now that we've introduced the SNE probabilities, we prove how it fits into the ProbDR framework. \\
In the case of SNE, we assume in the ProbDR framework,
\begin{align*}
    q(\A'|\Y) &= \prod_i^n \text{Categorical}(\A'_{i:} ; \Y) = \prod_i^n \prod_{j \neq i}^n [v_{ij}^S]^{\A_{ij}} \text{ and} \\
    p(\A' | \X) &= \prod_i^n \text{Categorical}(\A'_{i:} | \X) = \prod_i^n \prod_{j \neq i}^n [w_{ij}^S]^{\A_{ij}}.
\end{align*}
This leads to the KL of \cref{eqn:prob-dr-kl},
\begin{align*}
    \text{KL}(q(\A'|\Y)||p(\A' | \X)) &= \sum_{i} \text{KL}(q(\A'_{i:})||p(\A'_{i:} | \X)) \\
    &= \sum_{i} \sum_{j\neq i} v_{ij}^{S} \log  \frac{v_{ij}^{S}}{w_{ij}^{S}} = C_{SNE}.
\end{align*}
\end{proof}

\subsection{Remark on the direction of KL and notation}

Note that our notation (and only the notation) is flipped w.r.t. the (t-)SNE papers, i.e. we define the objective as $\text{KL}(q\|p)$ rather than $\text{KL}(p\|q)$, although the computation of the objective remains the same.

To see why, note that the objective of \cite{sne} looks like,
$$\text{KL}(\text{probabilities involving data} \| \text{probabilities involving latents}).$$
Noting that in typical variational models (such as VAEs and variational GPLVMs), the variational distributions are a function of data, and the model distributions are a function of latents (or parameters associated with the generative model), we propose that it is more natural to set the data based probabilities to $q$ and write the objective as $\text{KL}(q\|p)$ as we do here. Noting this was one of the main inspirations for this project, along with the observation that many circularly-specified modelling methodologies can be written as variational inference algorithms. Note further that we denote high dimensional observed data by $\Y$ and low dimensional embeddings by $\X$, taking inspiration from how regression models are typically specified, whereas many older works such as (t-)SNE use reversed notation.

\subsection{t-SNE}

The t-SNE algorithm was introduced by \cite{tsne} to improve optimisation and visualization. In the t-SNE algorithm, probabilities $v_{ij}^{t}$ and $w_{ij}^{t}$ are defined as,
\begin{align*}
    v_{ij}^{t} &= (v_{ij}^{S} + v_{ji}^{S})/2n, \\
    w_{ij}^{t} &= \frac{(1 + \|\X_{i:} - \X_{j:}\|^2)^{-1}}{\sum_{k\neq l} (1 + \|\X_{k:} - \X_{l:}\|^2)^{-1}},
\end{align*}
which are then matched by minimizing the cost function below w.r.t $\X$,
\begin{align*}
    \mathcal C_{t-SNE} = \sum_{i\neq j} v_{ij}^{t} \log  \frac{v_{ij}^{t}}{w_{ij}^{t}}.
\end{align*} \\

Note that the normalization here, as opposed to SNE, is over the entire set of probabilities.
\begin{proof} of \cref{thm:tsne-umap}, t-SNE case. Now that we've introduced the t-SNE probabilities, we prove how it fits into the ProbDR framework. Note that both sets of probabilities in t-SNE sum up to one. In the case of t-SNE, we assume in the ProbDR framework,
\begin{align*}
    q(\A'|\Y) &= \text{Categorical}(\text{vec}(\A') | \Y) = \prod_{i \neq j}^n [v_{ij}^{t}]^{\A_{ij}} \text{ and} \\
    p(\A' | \X) &= \text{Categorical}(\text{vec}(\A') | \X) = \prod_{i \neq j}^n [w_{ij}^{t}]^{\A_{ij}}.
\end{align*}
Therefore the KL of \cref{eqn:prob-dr-kl},
\begin{align*}
    \text{KL}(q(\A'|\Y)||p(\A' | \X)) &= \sum_{i\neq j} v_{ij}^{t} \log  \frac{v_{ij}^{t}}{w_{ij}^{t}} = C_{t-SNE}.
\end{align*}
\end{proof}

\subsection{UMAP}

The UMAP algorithm \citep{umap} has proven to be a popular choice  in computational biology for visualizing single-cell RNA-seq data due to decreased runtimes and a greater ability of recovering cell clusters as compared with t-SNE \citep{umap-nature}. The algorithm defines probabilities $v_{ij}^{U}$ and $w_{ij}^{U}$ as
\begin{align*}
    v_{i|j}^{U} &= \exp((\rho_i-\text{distance}(\Y_{i:}, \Y_{j:})) / \sigma_i), \\
    v_{ij}^{U} &= v_{i|j} + v_{j|i} - v_{i|j} * v_{j|i}, \\
    w_{ij}^{U} &= (1 + a\|\X_{i:} - \X_{j:}\|^{2b})^{-1},
\end{align*}
where $\rho_i$ denotes the distance to the nearest neighbour of data point $i$. These are matched by optimizing the cost function below w.r.t $\X$:
\begin{align*}
    \mathcal C_{UMAP} = \sum_{i\neq j} v_{ij}^{U} \log  \frac{v_{ij}^{U}}{w_{ij}^{U}} + (1-v_{ij}^{U}) \log  \frac{1 - v_{ij}^{U}}{1 - w_{ij}^{U}}.
\end{align*} \\

\begin{proof} of \cref{thm:tsne-umap}, UMAP case. Now that we've introduced the UMAP probabilities, we prove how it fits into the ProbDR framework. We use UMAP notation for defining probabilities (i.e. $v$ and $w$) throughout this paper. \\
In the case of UMAP, we assume in the ProbDR framework,
\begin{align*}
    q(\A'|\Y) &= \prod_{i<j}^n \text{Bernoulli}(\A'_{ij} ; \Y) = \prod_i^n \prod_{j<i}^n [v_{ij}^U]^{\A'_{ij}}[1-v_{ij}^U]^{1-\A'_{ij}} \text{ and} \\
    p(\A' | \X) &= \prod_{i<j}^n \text{Bernoulli}(\A'_{ij} | \X) = \prod_i^n \prod_{j<i}^n [w_{ij}^U]^{\A'_{ij}}[1-w_{ij}^U]^{1-\A'_{ij}}.
\end{align*}
Hence,
\begin{align*}
    \text{KL}(q(\A'|\Y)||p(\A' | \X)) &= \sum_{i} \sum_{j<i} \text{KL}(q(\A'_{ij})||p(\A'_{ij} | \X)) \\
    &= \sum_{i} \sum_{j<i} v_{ij}^{U} \log  \frac{v_{ij}^{U}}{w_{ij}^{U}} + (1-v_{ij}^{U}) \log  \frac{1 - v_{ij}^{U}}{1 - w_{ij}^{U}} = C_{UMAP}/2.
\end{align*}
\end{proof}

\subsection{Remarks on adjacencies, Laplacians and marginal variances}

\textbf{A note on adjacency matrices:} In models corresponding to SNE and t-SNE, the adjacency matrix we define ($\A'$) can be thought of as an adjacency matrix on a directed graph; the use of categorical distributions results in $\A'$ being asymmetric. In the UMAP case, the matrix $\A'$ also represents a directed graph as the lower triangle of the adjacency matrix is independent of the upper triangle (hence, samples of the matrix are likely to be asymmetric). However, by changing the range of the sum of the UMAP cost function to $i < j$ rather than summing over $i \neq j$ (which simply results in a factor of a half appearing before the cost function), one may construct a lower triangular adjacency matrix, describing a directed acyclical graph, which may be made symmetric as needed (as in \cref{section:probdr}). Adjacency matrices representing DAGs or more generally directed graphs can be used as part of other generative models described in \cref{app:alt-gen-models}. \\

\textbf{A note on marginal variances:} A covariance $\mathbf{C}$ can be normalized as,
$$\text{diag}(\mathbf{C})^{-1/2} \mathbf{C} \text{diag}(\mathbf{C})^{-1/2},$$
to make it a correlation matrix with uniform marginal variances. This process can be done approximately, efficiently and differentiably by using the eigendecomposition of the covariance or the precision.

\textbf{A note on the Laplacian as a precision:} The Laplacian as part of a precision matrix is meaningful as it is positive definite, describes non-negative partial correlations between data points (and where the partial correlation is zero, data points are conditionally independent), and is typically sparse, describing a sparsely-connected Gaussian field on the data.

\section{Two-Step MLE Proofs \& Wishart Results}
\label{app:two-step}

\subsection{Remark on notation used for Wishart matrices}

Many of our statements involving Wishart distributed random matrices are denoted as,
$$ \T \sim \mathcal{W}(\hat{\M}, d). $$
The random matrix without a hat or an overset tilda $\T$ represents a matrix that is scaled in some way, and matrices with a hat or overset tilda represent unscaled quantities. In this example, note that $\mathbb E(\T) = \hat{\M} * d$. Therefore, sample (estimated) covariances in our work are denoted as $\hat{\S}$, as they are typically calculated as $\Y\Y^T/d$, and hence are an unscaled quantity. In this example, we would denote by $\S$ the unscaled random matrix $\Y\Y^T$, which (assuming that the columns of $\Y$ are independent multivariate normal samples) by definition is Wishart distributed, and is scaled by the degrees of freedom in expectation.

\subsection{Summary of main result}

The main result stated in \cref{thm:two-step-mle} states that the MAP estimation for $\X$ given the models (with an improper uniform prior over $\X$) occurs at the $\q$ major and minor scaled eigenvectors respectively,
\begin{align*}
    \hat{\mathbf{\S}} * d | \mathbf{X} \sim \mathcal{W}\left(\X \X^T + \sigma^2 \I_n, d\right) \quad &\Rightarrow \quad \hat{\X}_{\text{MAP}} = \mathbf{U}_{\q \text{ maj}} (\mathbf{\Lambda}_{\q \text{ maj}} - \hat \sigma^2 \I_{\q})^{1/2} \mathbf{R}^T \\
    \hat{\mathbf{\Gamma}} * d | \mathbf{X} \sim \mathcal{W}\left((\X \X^T + \beta \I_n)^{-1}, d\right) \quad &\Rightarrow \quad \hat{\X}_{\text{MAP}} = \mathbf{U}_{\q \text{ min}} (\mathbf{\Lambda}^{-1}_{\q \text{ min}} - \hat \beta \I_{\q})^{1/2} \mathbf{R}^T \\
\end{align*}
where $\U_\q$ is a matrix of $\q-$eigenvectors, $\mathbf{\Lambda}_\q$ is a diagonal matrix of $\q$ corresponding eigenvalues and $\mathbf{R}$ is an arbitrary rotation matrix. $\q$ maj and min denote the eigencomponents corresponding to the $\q$-largest and lowest eigenvalues respectively. \\

Note that, in this work for simplicity, we generally assume that $\Y$ has a zero mean and doesn't need to be centered and that generative models for $\Y$ may be set to have a zero mean. \\

This is due to two key results, probabilistic \textbf{principal coordinate analysis} (of \cite{gplvm}, based on \cite{ppca}) and probabilistic \textbf{minor coordinate analysis} (which we derive, based on results of \cite{mca}).

\subsection{Probabilistic principal coordinate analysis based results}

First note that, for some applications, multivariate normal likelihoods and Wishart likelihoods are equal, up to additive constants.

\begin{lemma}
\label{thm:wish-models}
Let $\mathbf{F} \in \mathbb R^{n \times d}$ and $\T \equiv \Tilde{\T} * d \equiv \mathbf{F}\mathbf{F}^T$. The log likelihood of the following models is equal up to additive constants that do not depend on $\M$,
\begin{align*}
    \log p(\mathbf{F} | \M) \text{ assuming } \mathbf{F} | \M &\sim \mathcal{MN}(0, \M, \I_d) \text{ and }, \\
    \log p(\T | \M) \text{ assuming } \T | \M &\sim \mathcal{W}\left(\M, d \right).
\end{align*}
\end{lemma}

\begin{proof} of \cref{thm:wish-models}.

 In the normal case,
\begin{align*}
    \mathcal L(\mathbf{F}) = \log p(\mathbf{F} | \M) &= -\dfrac{1}{2} \text{tr}\left( \I_d \mathbf{F}^T \M^{-1} \mathbf{F} \right) - \dfrac{d}{2} \log |\M| -  \dfrac{n}{2} \log |\I_d| - \dfrac{np}{2} \log 2\pi \\
    &= -\dfrac{\red{d}}{2} \text{tr}\left(\red{\dfrac{1}{d}} \mathbf{F} \mathbf{F}^T \M^{-1} \right) - \dfrac{d}{2} \log |\M| + c, \text{\hspace{1.5cm} (trace is cyclic)} \\
    &= -\dfrac{d}{2} \text{tr}\left(\Tilde{\T} \M^{-1} \right) - \dfrac{d}{2} \log |\M| + c.
\end{align*}

In the Wishart case when $d \geq n$, the sampling distribution of $\mathbf{F}\mathbf{F}^T$ is by definition Wishart, so the likelihood w.r.t. $\Tilde{\S}$ can be obtained easily,
\begin{align*}
    \mathcal L(\mathbf{F}) = \log p_{\mathcal W}(\Tilde{\T} * d | \M) = -\dfrac{d}{2} \text{tr}\left( \M^{-1}\Tilde{\T} \right) - \dfrac{d}{2} \log |\M| + c.
\end{align*}

In the case when $d < n$, the distribution of $\mathbf{F}\mathbf{F}^T$ is a singular Wishart \citep{singular-wishart}. The likelihood can be computed using Theorem 6 of \cite{singular-wishart}, and is identical to the statement above up to additive constants.
\begin{align*}
    \mathcal L(\mathbf{F}) = \log p_{\mathcal W}(\Tilde{\T} * d | \M) = -\dfrac{d}{2} \text{tr}\left( \M^{-1}\Tilde{\T} \right) - \dfrac{d}{2} \log |\M| + c.
\end{align*}
\end{proof}

\begin{lemma}[Probabilistic Principal Coordinates Analysis (PCA)] The maximum likelihood estimate of $\X$ assuming the model,
$$ \mathcal{N}(\Y | 0, \X\X^T + \sigma^2 I) \text{ or } \mathcal{W} (\S | \X\X^T + \sigma^2 I, d) $$
where $\S \equiv \Tilde{\S} * d = \Y\Y^T$ and the corresponding optimisation is,
$$ \argmax_{\X} \hspace{0.2cm} - \frac{d}{2} \log |\C| - \frac{d}{2} \text{tr}(\Tilde{\S}\C^{-1}) + c, $$
occurs at,
$$ \hat{\X} = \mathbf{U}_{\q} (\mathbf{\Lambda}_\q - \hat{\sigma}^2 \I_{\q})^{1/2} \mathbf{R}^T, $$
where $\hat{\sigma}^2 = \frac{\sum_{i=\q+1}^{n} \lambda_i}{n - \q}$ and $\mathbf{U}_\q \text{ and } \mathbf{\Lambda}_\q$ are the matrices of $\q$ major eigenvectors and eigenvalues of $\Tilde{\S}$.
\label{thm:pca}
\end{lemma}

\begin{proof} of \cref{thm:pca}. \\

The Wishart model is equivalent to the normal case due to \cref{thm:wish-models}. The main result is due to \cite{gplvm}, which is based on \cite{ppca}.
\end{proof}

This proves the first claim of \cref{thm:two-step-mle}. \\

\textbf{Remark on equivalent inverse wishart statements:} Wishart distributions and inverse-Wishart distributions are closely tied,
\begin{equation*}
    \mathbf{W} \sim \mathcal{W}(\mathbf{M}, \rho) \Leftrightarrow \mathbf{W}^{-1} \sim \mathcal{W}^{-1}(\mathbf{M}^{-1}, \rho),
\end{equation*}
and so many of the Wishart sampling statements can also be written instead with inverse-Wisharts.

\subsection{Probabilistic minor coordinate analysis based results}
\label{app:sub:mca}

In this section, we will prove the second Wishart statement of \cref{thm:two-step-mle}. We do so by first describing a novel probabilistic dimensionality reduction model, \textbf{probabilistic minor coordinates analysis}. \\

\begin{theorem}[Minor Coordinates Analysis (MCA)]
\label{thm:mca}
We propose a dimensionality reduction method utilising the result of probabilistic minor components analysis \citep{mca}. Using this algorithm, and given an estimated/empirical precision matrix $\Tilde{\mathbf{\Gamma}}$, we find a low dimensional embedding $\X$ by maximising objectives of the form below. \\

Let $\X \in \mathbb R^{n \times \q}$ and $\mathbf{\Gamma} \equiv  \Tilde{\mathbf{\Gamma}} * d$. Then,
\begin{align*}
    \argmax_{\X} \hspace{0.2cm} \frac{d}{2} \log |\mathbf{P}^{-1}| - \frac{d}{2} \text{tr}(\mathbf{P}^{-1} \Tilde{\mathbf{\Gamma}}) + c
\end{align*}
with $\mathbf{P}^{-1} \equiv \X\X^T + \beta \I_{n} $ is attained at,
$$ \hat{\X} = \mathbf{U}_{\q} (\mathbf{\Lambda}^{-1}_\q - \hat \beta \I_{\q})^{1/2} \mathbf{R}^T, $$
where $\hat \beta = \frac{n - \q}{\sum_{i=\q+1}^{n} \lambda_i}$ and $\mathbf{U}_\q \text{ and } \mathbf{\Lambda}_\q$ are the matrices of $\q$ minor eigenvectors and eigenvalues of $\Tilde{\mathbf{\Gamma}}$.
\label{thm:mca}
\end{theorem}

\begin{proof} of \cref{thm:mca}

The result is based on the result of \cite{mca} if one starts with the notation $\Tilde{\mathbf{\Gamma}}$, $\mathbf{P}^{-1}$, $\X$, $n$, $d$, $\q$ instead of $\S$, $\mathbf{C}^{-1}$, $\mathbf{W}$, $d$, $N$, $m$. \\

More explicitly, it's been shown in \cite{mca}, that the maximum likelihood estimate of the parameter $\blu{\mathbf{W}} \in \mathbb R^{\blu{d} \times \blu{m}}$, in objectives of the form below,
$$ \mathcal L = \frac{\blu{N}}{2} \log |\blu{\mathbf{C}}^{-1}| - \frac{\blu{N}}{2} \text{tr}(\blu{\mathbf{C}}^{-1} \blu{\S}) + c, $$
with $\blu{\mathbf{C}}^{-1} \equiv \blu{\mathbf{W}\mathbf{W}^T} + \beta \I_{\blu{d}}$, is,
$$ \hat{\blu{\mathbf{W}}} = \mathbf{U}_{\blu{m}} (\mathbf{\Lambda}^{-1}_{\blu{m}} - \hat \beta \I_{\blu{m}})^{1/2} \mathbf{R}^T, $$
where $\hat \beta = \frac{\blu{d} - \blu{m}}{\sum_{i=\blu{m}+1}^{\blu{d}} \lambda_i}$ and $\mathbf{U}_{\blu{m}} \text{ and } \mathbf{\Lambda}_{\blu{m}}$ are the matrices of $\blu{m}$ minor eigenvectors and eigenvalues of $\blu{\S}$. Key notation has been highlighted {\color{blue} in blue}. If notation is changed as follows, 
$\blu{\S} \rightarrow \Tilde{\mathbf{\Gamma}}$, $\blu{\mathbf{C}^{-1}} \rightarrow \mathbf{P}^{-1}, \blu{\mathbf{W}} \rightarrow \X$, $\blu{d} \rightarrow n$, $\blu{N} \rightarrow d$, $\blu{m} \rightarrow \q$, the proposed statement follows.
\end{proof}

The probabilistic interpretation of this is trivial and is laid out below.

\begin{lemma}[Probabilistic Minor Coordinates Analysis] Minor coordinates analysis is maximum likelihood inference given the model,
$$
    \mathbf{\Gamma} | \mathbf{X} \sim \mathcal{W}\left((\X \X^T + \beta \I_n)^{-1}, d\right)
$$
where $\Tilde{\mathbf{\Gamma}} \equiv \mathbf{\Gamma}/d$ is an empirical precision matrix, for example, calculated as $(\Y\Y^T/d)^{-1}$.
\label{thm:pmca}
\end{lemma}

\begin{proof} of \cref{thm:pmca}. \\

The objective in \cref{thm:mca} is the likelihood of the models in \cref{thm:wish-models} with $\Tilde{\T} = \Tilde{\mathbf{\Gamma}}$.
\end{proof}

Probabilistic minor coordinates analysis is the second statement of \cref{thm:two-step-mle}, hence completing the proof of our main statement.

\subsection{ProbDR \& 2-step MAP equivalence}
\label{app:probdr-map-equiv}

Now, we show that the two step MAP inference process is equivalent to ProbDR.

\begin{theorem}[ProbDR KL minimisation and MAP Equivalence: Wishart Case]
The Maximum a posteriori estimate for $\X$, i.e. $\argmax_{\X} \log p(\hat{\M}(\Y)*d| \X)$ assuming
$$ p(\M | g(\X)) = \mathcal{W}(\M | g(\X), d) $$ and an improper uniform prior
is equivalent to finding $\argmin_{\X} \text{KL}( q(\M | \hat{\M}(\Y)) \| p(\M | g(\X))))$ in the variational setup,
\begin{align*}
    \text{model (law of p)}: \M | g(\X) &\sim \mathcal{W}(g(\X), d), \\
    \text{variational approx (law of q)}: \M | \hat{\M}(\Y) &\sim \mathcal{W}(\hat{\M}(\Y), d).
\end{align*}
\label{thm:weird_vi_mle}
\end{theorem}

\begin{proof} of \cref{thm:weird_vi_mle}.

In the maximum likelihood setup, the negative log likelihood is as follows,
$$ -\log p(\hat{\M}(\Y)*d| \X) = \dfrac{d}{2} \text{tr}(g(\X)^{-1} \hat{\M}(\Y)) + \frac{d}{2} \log |g(\X)|. $$

Using the result of KL divergence between two Wishart distributions, the variational bound can be written as,
$$\text{KL}( q(\M | \hat{\M}(\Y)) \| p(\M | \X))) = \frac{d}{2} \left( \log |g(\X)| - k \right) + \dfrac{d}{2}\text{tr}(g(\X)^{-1}\hat{\M}(\Y)) + c.$$

The bounds are equal up to additive constants.
\end{proof}

Such a result is true for many exponential family (see \cite{expfam} for a definition) distributions, as for exponential family densities $p$ and $q$, where $q$ has no parameters of interest,
\begin{align*}
    -\text{KL}(q \| p) &= -\mathbb E_q(\log q(\mathbf{x}) / \log p(\mathbf{x})) \\
    &= -[\eta(\theta_q) - \eta(\theta_p)]^T \cdot \mathbb E_q(\mathbf{T(x)}) + [A(\eta_q) - A(\eta_p)] \\
    &= \eta(\theta_p)^T  \cdot \mathbb E_q(\mathbf{T(x)}) - A(\eta_p) + c,
\end{align*}
which is the log likelihood of the exponential family distribution $p$, up to a constant, with the expectation of the sufficient statistic under the variational distribution being set to the observed sufficient statistic. \\

Next, we show some consistency results.

\subsection{Consistency of PCA \& MCA, and ProbDR marginal consistency proofs}
\label{app:sub:consistency}

Firstly, we show that pPCA and pMCA obtain the same solution, even though they are not equivalent statements.

\begin{lemma}[Equivalence of probabilistic PCA and probabilistic MCA] \\
Let $\Tilde{\S} \equiv \S/d \equiv \Y\Y^T/d $ and $ \Tilde{\mathbf{\Gamma}} \equiv \mathbf{\Gamma} / d \equiv \Tilde{\S}^{-1} $. The matrices $\Tilde{\S} \text{ and } \Tilde{\mathbf{\Gamma}}$ share eigenvectors, represented by the matrix $\U$ and their diagonal eigenvalue matrices are $\mathbf{\Lambda}_{\tilde{\S}}$ and $\mathbf{\Lambda}_{\tilde{\S}}^{-1}$ respectively. Then, the estimated covariance assuming either of the following models,
\begin{align*}
    \S | \X &\sim \mathcal{W}\left( \X\X^T + \sigma^2 \I, d \right) \\
    \mathbf{\Gamma} | \X &\sim \mathcal{W}\left((\X\X^T + \beta \I)^{-1}, d \right) \\
\end{align*}
is identical.
\label{thm:mca_pca_equiv}
\end{lemma}
\begin{proof} of \cref{thm:mca_pca_equiv}. \\
The proof is due to \cref{thm:pca} and \cref{thm:pmca}. In PCA,
$$ \hat{\S}_{\text{PCA}} = \hat{\X}\hat{\X}^T + \hat{\sigma}^2 \I = \U \mathbf{\Lambda}_{\hat{\S}} \U^T, $$
and in MCA,
$$ \hat{\S}_{\text{MCA}} = \hat{\X}\hat{\X}^T + \hat{\beta} \I = \U \mathbf{\Lambda}^{-1}_{\hat{\mathbf{\Gamma}}} \U^T = \U \mathbf{\Lambda}_{\hat{\S}} \U^T = \hat{\S}_{\text{PCA}}. $$
\end{proof}

Note that, although the covariance estimates are the same, the noise levels are not.

\begin{theorem}[Estimated noise level in PCA and MCA]
    The estimated noise level in MCA $\hat{\beta}$ is lower than its counterpart in PCA $\hat{\sigma}^2$,
    $$ \hat{\beta} \leq \hat{\sigma}^2. $$
\label{thm:beta_vs_sigma_sq}
\end{theorem}
\begin{proof} of \cref{thm:beta_vs_sigma_sq}. \\
Assume the setup of \cref{thm:mca_pca_equiv} and let $\lambda$s be major eigenvalues of the sample covariance matrix.

Due to \cref{thm:pca}, \cref{thm:mca}, and due to the fact that the major eigenvalues of the sample covariance are minor eigenvalues of the precision,
$$ \hat{\sigma}^2 = \frac{\sum_{i=\q+1}^{n} \lambda_i}{n - \q} \text{ and } \hat{\beta} =  \frac{n - \q}{\sum_{i=\q+1}^{n} \dfrac{1}{\lambda_i}}. $$

Therefore,
\begin{align*}
    \dfrac{\hat{\sigma}^2}{\hat{\beta}} &= \dfrac{ \sum_{i=\q+1}^{n} 1/\lambda_i \sum_{i=\q+1}^{n} \lambda_i }{(n - \q)^2} \overset{\text{AM-GM}}{\geq} \sqrt[n-\q]{\prod_i \lambda_i/\lambda_i} = 1.
\end{align*}
\end{proof}

Below, we show that marginalising the moment in either case of our Wishart models leads to the standard Gaussian process assumed in many linear DR models (i.e. one with a dot product kernel). We do not show how column independence arises, although this is trivial plugging in $\text{vec}(\Y)$ into the results below and using the matrix normal distribution's definitions.

\begin{theorem}[Marginal consistency with PCA]
Assuming a PCA-esque generative model,
\begin{align*}
    \mathbf{y} | \S &\sim \mathcal{N}\left(\mathbf{0}, \dfrac{1}{\rho} \S \right), \\
    \S | \X &\sim \mathcal W(\X\X^T + \sigma^2 \mathbf{I}, \rho),
\end{align*}
or a MCA-esque generative model,
\begin{align*}
    \mathbf{y} | \S &\sim \mathcal{N}\left(\mathbf{0}, \S * (\rho - n + 1) \right), \\
    \S | \X &\sim \mathcal{W}^{-1}\left(\X\X^T + \beta \I, \rho \right)
\end{align*}
the marginal distribution of any column of the data $\mathbf{y}$, as $\rho \rightarrow \infty$, is given by,
\begin{align*}
    \mathbf{y} | \X \sim \mathcal N(\mathbf{0}, \mathbf{X} \mathbf{X}^T + \sigma^2 \mathbf{I}).
\end{align*}
\label{thm:pca_marginal}
\end{theorem}

\begin{proof} of \cref{thm:pca_marginal}.

In the first case,
$$\text{Var}\left(\dfrac{\S_{ij}}{\rho}\right) = \dfrac{\left([\X\X^T]_{ij}^2 + [\sigma^2 + \X\X^T]_{ii}[\sigma^2 + \X\X^T]_{jj}\right)_{ij}\rho}{\rho^2} \rightarrow 0 \text{ and,}$$
$$\mathbb E\left(\dfrac{\S}{\rho}\right) = \X\X^T + \sigma^2 \I.$$
Therefore, $\S/\rho$ converges to a constant matrix, hence the marginal in the limit is $\mathcal N(\mathbf{0}, \X\X^T + \sigma^2 \I)$. In the second case, due to conjugacy \citep{pml-ii},
\begin{align*}
    \mathbf{y} | \X \sim t_{\rho-n+1}(\mathbf{0}, \mathbf{X} \mathbf{X}^T + \sigma^2 \mathbf{I}),
\end{align*}
which tends to the normal statement above as $\rho \rightarrow \infty$.
\end{proof}

\textbf{A note on Wishart-Normal conjugacy:} Some common references state normal conjugacy results using atypical notation for Wishart distributions, so we prove the result above, using notation used in this paper, from first principles. Let $\mathbf{y} \sim \mathcal{N} (\mathbf{0}, \kappa\S)$ and $\S \sim \mathcal W^{-1}(\M, d)$. Then,
\begin{align*}
    p(\y) &= \int p(\y | \S) p(\S) d\S \\
    &\propto |\M|^{d/2} \int |\S|^{-(d + n + 2)/2} \exp(-\kappa^{-1}\y^T \S^{-1} \y/2 - \text{tr}(\M\S^{-1})) d\S \\
    &\propto |\M|^{d/2} \int |\S|^{-(d + n + 2)/2} \exp(- \text{tr}(\kappa^{-1}\y\y^T \S^{-1})/2 - \text{tr}(\M\S^{-1})) d\S \\
    &\propto |\M|^{d/2} \int |\S|^{-(d + n + 2)/2} \exp \left[- \text{tr} \Big( (\kappa^{-1}\y\y^T + \M)\S^{-1}) \Big) \right] d\S \\
    &\overset{\int p=1}{\propto} |\M|^{d/2} |\kappa^{-1}\y\y^T + \M|^{-(d + 1)/2} \\
    &\propto |\M|^{d/2} |\y\y^T + \kappa\M|^{-(d + 1)/2} \\
    &\propto |\M|^{-1/2} \left[ 1 + \frac{1}{\kappa}\y^T \M^{-1} \y \right]^{-(d+1)/2} \qquad \text{ (matrix determinant lemma)} \\
    &\propto t_{d - n + 1}\left(\y| \mathbf{0}, \frac{\kappa\M}{d - n + 1} \right).
\end{align*}

\subsection{Utility results (for PPLs and distributions of normal distances)}

PPLs generally do not have support for degenerate Wishart distributions. Therefore, we use the trick below to deal with low-$d$, high-$n$ problems.

\begin{lemma}
Let $\mathbf{F} \in \mathbb R^{n \times d}$ and $\Tilde{\T} \equiv \mathbf{F}\mathbf{F}^T/d$ with $d < n \leq \rho$. The following computation results in the log-likelihood of \cref{thm:wish-models} up to additive \textbf{and multiplicative} constants,
$$
    \log \mathcal{W}\left(\rho \Tilde{\T} | \M, \rho \right).
$$
This is useful for usage with PPLs, \textbf{when doing MAP inference}, with no support for singular Wishart distributions - note that due to the multiplicative constant, models specified in the PPL must not contain any other sampling statements, or if they are present, the corresponding likelihoods must be appropriately weighted. Unlike in the case presented here, multiplicative constants do matter however, when doing MCMC sampling.
\label{thm:wish-ppl}
\end{lemma}
\begin{proof} of \cref{thm:wish-ppl}.
\begin{align*}
    \mathcal L(\tilde{\T}) = -\dfrac{\rho}{2} \text{tr}\left( \M^{-1}\Tilde{\T} \right) - \dfrac{\rho}{2} \log |\M| + c,
\end{align*}
which is equal to the likelihood of \cref{thm:wish-models} up to a factor of $d/\rho$ and an additive constant.
\end{proof}

Below, we show that the Categorical/Bernoulli (i.e. (t-)SNE and UMAP) cases of the ProbDR framework are approximately equivalent to two step MAP. This is useful, for example, if one attempts to use a PPL for DR with limited support for variational inference.
\begin{lemma}[(t)-SNE/UMAP ProbDR is approximately two-step MAP]
\label{thm:probdr-is-mle}
\end{lemma}
\begin{proof} of \cref{thm:probdr-is-mle}.
As our variational distributions do not have any optimised parameters,
\begin{align*}
    \text{KL}_{\text{categorical}}( q \| p) &= \sum_i q_i \log \left( \dfrac{q_i}{p_i} \right) \\
    &= -\sum_i q_i \log p_i + c \\
    &\approx -\dfrac{1}{n} \sum_i \lfloor n q_i \rfloor \log p_i + c.
\end{align*}
So, the KL divergence of ProbDR in the (t-)SNE cases (i.e. with a categorical distribution on the probabilities of adjacency) is approximately equal to the negative log likelihood (up to an additive and multiplicative constant) of a categorical distribution with the observed random variables set to $\lfloor n q_i \rfloor$ for category (data point) $i$.
The case for the Bernoulli follows as it is a special case of the categorical distribution.
\end{proof}

Finally, for future work (for example, if graphs in the ProbDR-UMAP case are to be analysed as mixtures of graphs / hypergraphs arising due to kernel choice in a GP on a manifold), we note a simple result below that can be used to calculate adjacency probabilities.
\begin{lemma}[Distribution of normal distances]
$$\begin{bmatrix} \mathbf{y}_i \\ \mathbf{y}_j \end{bmatrix} \sim \mathcal{MN} \left(\mu, \begin{bmatrix} k_{ii} & k_{ij} \\ k_{ji} & k_{jj} \end{bmatrix}, \I_d \right) $$
$$ \Rightarrow \| \mathbf{y}_i - \mathbf{y}_j\|^2 \sim \Gamma\left(\dfrac{d}{2}, 2(k_{ii} + k_{jj} - 2k_{ij})\right).$$
\label{thm:norm_dist}
\end{lemma}
\begin{proof} of \cref{thm:norm_dist}.
\begin{align*}
    \forall k: y_i^k - y_j^k &\sim \mathcal{N} (0, k_{ii} + k_{jj} - 2k_{ij}) \overset{d}{=} \sqrt{k_{ii} + k_{jj} - 2k_{ij}}Z \\
    \Rightarrow \| \mathbf{y}_i - \mathbf{y}_j\|^2 = \sum_k^d (y^k_i - y^k_j)^2 &\overset{d}{=} (k_{ii} + k_{jj} - 2k_{ij}) \sum_k^d Z_k^2 \\
    &\overset{d}{=} (k_{ii} + k_{jj} - 2k_{ij}) \chi^2_d \\
    &\overset{d}{=} \Gamma(k=d/2, \theta=2(k_{ii} + k_{jj} - 2k_{ij})).
\end{align*}
\end{proof}

\subsection{Equivalence of GPLVMs and ProbDR}

Inference in classical GPLVMs \citep{gplvm} occurs by maximising the log-likelihood,
$$ \log \mathcal{MN}(\Y| \mathbf{0}, K_{\theta}(\X), \I). $$
This is equivalent, due to \cref{thm:weird_vi_mle}, to $\text{KL}(q(\S | \hat{\S}) \| p(\S | K(\X)))$ assuming,
\begin{align*}
    p(\S | K_{\theta}(\X)) &= \mathcal{W}(\S | K_{\theta}(\X), d), \\
    q(\S | \hat{\S}) &= \mathcal{W}(\S | \Y\Y^T/d, d).
\end{align*}
Note that a very similar KL minimisation also appears in \cite{gplvm}.

\subsection{A probabilistic interpretation of DRTree}

The objective that the DRTree \citep{drtree} algorithm is based on, can be written as follows,
\begin{align*}
    \mathcal{L} &= \| \Y - \mathbf{W} \X \|^2 + \dfrac{\lambda}{2} \sum_{ij} b_{ij} \| \mathbf{W} \X_{i:} - \mathbf{W} \X_{j:} \|^2 \\
    &= \text{tr}( (\Y - \X\mathbf{W})^T(\Y - \X\mathbf{W}) ) + \text{tr}(\lambda \L \X \X^T).
\end{align*}
where the second step is due to the results in \cite{neil-gmrf} and the fact that $\mathbf{W}\mathbf{W}^T$ is constrained to be $\I$. This objective is approximately a negative log posterior assuming,
\begin{align*}
    \Y | \X &\sim \mathcal{MN}(\X\mathbf{W}, \I, \I) \\
    \X | \L &\sim \mathcal{MN}(\mathbf{0}, (\lambda \L + \beta \I)^{-1}, \I) \\
    \L &\sim \text{Uniform}_{\text{graph Laplacians over trees}}
\end{align*}
for small $\beta$ and such that $\mathbf{W}\mathbf{W}^T = \I$. The optimisation occurs w.r.t. $\X, \mathbf{W}$ and $\L$. To optimize over $\L$ given the other parameters, as in \cite{drtree}, one must use Kruskal's algorithm. It's interesting to note that this is more akin to traditional LVMs (as in \cref{eqn:gp}) but with an interesting prior over the latents, constraining them to be tree structured via a graph covariance on a tree.

\section{Experimental details}
\label{app:expts}

\subsection{\cref{fig:expt_a}}

\cref{fig:expt_a} shows that a PPL such as Stan \citep{stan} can be used for DR via automated MAP inference corresponding to models specified using the appropriate ProbDR interpretation. The data used for this experiment was a ten data point subset of MNIST, with each digit augmented with twenty-five rotations. This provides a high-$d$, low-$n$ dataset. We compare against popular open-source implementations (of \texttt{umap-learn} and \texttt{scikit-learn}). \cref{fig:stan-code} shows the Stan program written for this experiment. \\

\begin{figure}
    \centering
    \includegraphics[width=\columnwidth]{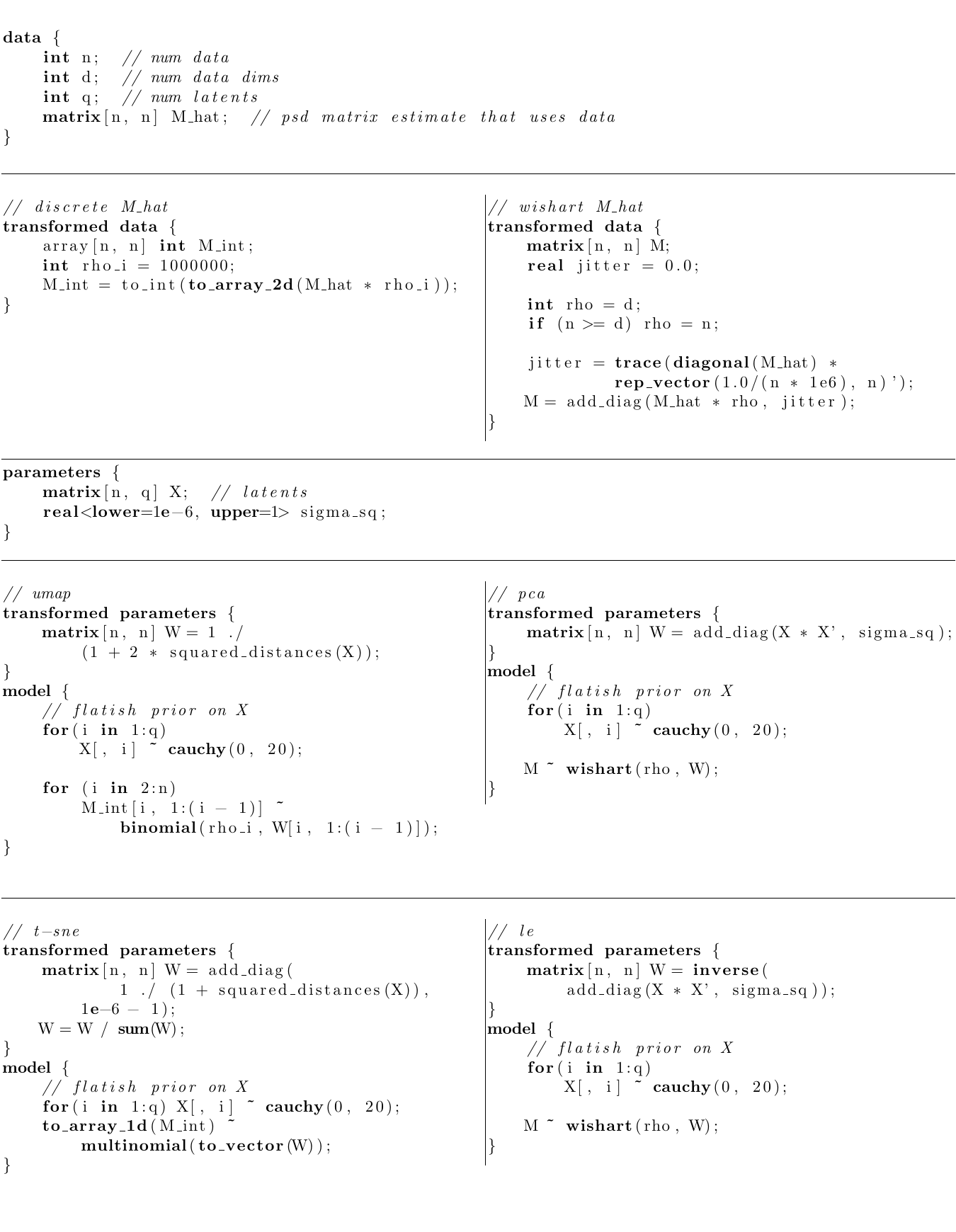}
    \caption{Stan code used for \cref{fig:expt_a}.}
    \label{fig:stan-code}
\end{figure}

\subsection{\cref{fig:expt_b}}

\cref{fig:expt_b} shows how one can predict at unseen locations using the ProbDR framework. The dataset used was the mouse brain cell RNA-seq transcriptomics dataset of \cite{mouse-data}. We used 50\% of the Lamp5 cluster for training and 50\% as the unseen test dataset (which results in about a thousand data points in each case). The data is composed of 3000 highly variable genes (thus, this is the data dimension). The data, $\Y_{\text{train}}$ and $\Y_{\text{test}}$ correspond to gene expression (logCPM). First, we use a community implementation of UMAP with default hyperparameters to obtain $\X_{\text{train}}$, and many such implementations can also embed $\X_{\text{test}}$ (typically by fixing $\X_{\text{train}}$ and by optimising a likelihood or cost function w.r.t. $\X_{\text{test}}$, as in \cite{lalchand2022generalised}). These implementations also output distributions over $\L_{\text{train}}$ and $\L_{\text{train}}$. as they are computed in a straightforward manner from the data and embeddings. \\

Then, we train hyperparameters (lengthscale $\kappa$, scale $\sigma_s$ and noise level $\sigma_n$) in the observation model,
$$ \Y_{\text{train}} | \Tilde{\L}_{\text{train}} \sim \mathcal{MN} \left(\mathbf{0}, \sigma_s^2 \left[\Tilde{\L}_{\text{train}} + \dfrac{2}{\kappa^2}\I \right]^{-1} + \sigma_n^2 \I, \I \right), $$
by optimising the ProbDR ELBO (\cref{eqn:probdr-obj}) w.r.t. to these parameters. The only term that contributes non-constants to the ELBO is,
$$ \mathcal{L} = \mathbb{E}_{q(\Tilde{\L}_{\text{train}}|\Y_{\text{train}})}[\log p_{\kappa, \sigma_n, \sigma_s}(\Y_{\text{train}}|\Tilde{\L}_{\text{train}})]. $$
Note that we use a normalised graph Laplacian above, as without it, the variational and model graph statistics are typically very different. Despite the fact that we don't have marginal consistency, the learned hyperparameters seem to be fine for usage with the augmented model,
$$ \begin{bmatrix}
    \Y_{\text{train}} \\
    \Y_{\text{test}}
\end{bmatrix}  | \Tilde{\L}_{\text{full}} \sim \mathcal{MN} \left(\mathbf{0}, \begin{bmatrix}
    C_{\text{train}} + \sigma_n^2\I & C_{\text{cross}} \\
    C_{\text{cross}}^T & C_{\text{test}}
\end{bmatrix}, \I \right), $$
where $\C$ is the corresponding (block of the) covariance matrix. The predictions can be computed as,
$$ \mathbb E(\Y_{\text{test}} | \Y_{\text{train}}) = \C_{\text{cross}} (\C_{\text{train}} + \sigma_n^2 \I)^{-1} \Y_{\text{train}}.$$
A similar treatment of unseen data in a semi-supervised setting appears in \cite{gf-gp}.

This model achieves a better performance than a VAE trained on the same data. We believe that superior performance can be attributed to the ability of UMAP to cluster similar cells more accurately than a vanilla VAE. \\


\subsection{\cref{fig:expt_c}}

\cref{fig:expt_c} shows prior samples of the ProbDR generative model (with UMAP assumptions). We sample a one dimensional $\X$ vector (uniformly), and construct an adjacency matrix using the UMAP adjacency probability equation. Then, we sample a one dimensional $\Y$ vector using a Matern-$\infty$ graph Gaussian process and plot samples against our sampled $\X$. Note that the samples are reminiscent of GP prior samples, suggesting that the $\X \rightarrow \L \rightarrow \Y$ construction may approximate a GP on a manifold. Some justification for these ideas is given by the result of \cite{belkin} on the graph Laplacian -- Laplace-Beltrami operator convergence (i.e. the discrete graph Laplacian, under some assumptions, is an approximation of the continuous operator, eigenfunctions of which are used for constructing smooth kernels on manifolds \citep{manifold-gp}).
\begin{figure}
    \centering
    \includegraphics[width=0.5\columnwidth]{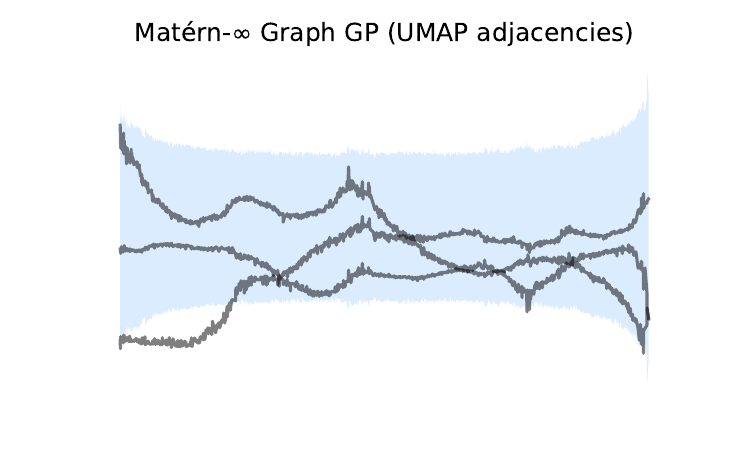}
    \caption{
    Samples of $\y$ plotted against $\X$, using a generative model, \\
    $\y | \Tilde{\L} \sim \mathcal N(\mathbf{0}, \exp[-12.5\L]),$ \\
    $\Tilde{\L} | \A' = \I - \mathbf{D^{\dagger}}^0.5 (\A' + \A'^T) \mathbf{D^{\dagger}}^0.5,$ \\
    $\forall i < j, \A_{ij}' | \X \sim \text{Bernoulli}(1/(1 + 2\| \X_{i:} - \X_{j:} \|^2)) $ and \\
    $\X \sim \text{Uniform}(-3, 3).$
    }
    \label{fig:expt_c}
\end{figure}

\section{A mean-field EM perspective on the UMAP Generative Model}
\label{app:em}

\subsection{Model set-up}

Assuming the model,
\begin{align*}
    p(\mathbf{A'}|\mathbf{X}) &= \prod_{i<j} \text{Bernoulli}(\A_{ij}'| \mathbf{\pi}_{ij}) \\
    \mathbf{A} &= \mathbf{A'} + \mathbf{A'}^T\\
    p(\mathbf{Y}|\mathbf{A}) &= \mathcal N(\mathbf{Y}| 0, (\mathbf{L} + \beta \mathbf{I})^{-1})
\end{align*}
The likelihood in this model can be written as:
\begin{align}
\label{eqn:em-like}
p(\mathbf{Y} | \mathbf{A}) = \frac{\left| \beta \mathbf{I} + \mathbf{L}\right|^{\frac{d}{2}}}{\left(2\pi\right)^{\frac{dn}{2}}} \exp\left(\frac{1}{2}\text{Tr}\left(\mathbf{Y}\mathbf{Y}^\top \left(\beta\mathbf{I} + \mathbf{L}\right)\right)\right) 
\end{align}

where $\mathbf{Y} \in \mathbb{R}^{n \times d}$ is our data in the form of a design matrix with $p$ features and $n$ data points. Define:
$$
\mathbf{L} = \rho\boldsymbol{\Phi}\mathbf{A_d}\boldsymbol{\Phi}^\top,
$$ 
where
$$
\boldsymbol{\Phi} = \left(\mathbf{1}^\top \otimes \mathbf{I} - \mathbf{I} \otimes \mathbf{1}^\top \right) \in \mathbb{R}^{n \times n^2}
$$
and 
$$
\mathbf{A_d} \in \mathbb{R}^{n^2 \times n^2}
$$
and has diagonal elements, where the $k$th diagonal is given by $\A_{ij}$ where $k = i + n(j-1)$, and we constrain $\A_{ii} = 0$ and $\A_{ij} = \A_{ji}$ and $\A_{ij}$ is either zero or one.

We introduce a mean-field variational distribution on $\mathbf{A}$:
$$q(\mathbf{A}) = \prod_{i<j}q(\mathbf{A}_{ij}).$$
Following \cite{blei-vi}, as part of mean-field variational inference, for every edge $ij$, we set $q(\mathbf{A}_{ij})$ proportional to:
$$
q\left(\mathbf{A}_{i j}\right) \propto \exp[ \mathbb E_{\mathbf{A_{-ij}}}( \log p(\mathbf{Y} \mid \mathbf{A})+\log p(\mathbf{A})) ].
$$
To write out these probabilities as an expectation given edges apart from $ij$, we formulate the determinant and trace terms in \cref{eqn:em-like}, 

\subsection{Calculating the Determinant}
The determinant can be calculated as:
$$
\begin{gathered}
|\beta \mathbf{I}+ \mathbf{L}|^\frac{d}{2}=\left| \beta \mathbf{I}+\rho\boldsymbol{\Phi} \mathbf{A} \boldsymbol{\Phi}^{\top}\right|^\frac{d}{2}
\end{gathered}
$$
$$
=\left| \beta \mathbf{I} +  \hat{\mathbf{L}}_{i j}+\A_{i j} \rho \boldsymbol{\phi}_{:, ij} \boldsymbol{\phi}_{:, ij}^{\top}+\A_{ji} \rho\boldsymbol{\phi}_{:, ji} \boldsymbol{\phi}_{:, ji}^{\top}\right|^{\frac{d}{2}}
$$
$$
=\left|\beta \mathbf{I}+\hat{\mathbf{L}}_{ij}\right|^{\frac{d}{2}}\left(1+2 \A_{i j} \rho\boldsymbol{\phi}_{:, ij}{ }^{\top}\left(\beta \mathbf{I}+\hat{\mathbf{L}}_{ij}\right)^{-1} \boldsymbol{\phi}_{:, ij}\right)^{\frac{d}{2}}
$$
$$
=\left|\beta \mathbf{I}+\hat{\mathbf{L}}^{ij}\right|^{\frac{d}{2}}\left(1+2  \A_{i j} \rho\boldsymbol{\phi}_{:, ij}{ }^{\top}\hat{\mathbf{C}}^{ij} \boldsymbol{\phi}_{:, ij}\right)^{\frac{d}{2}}
$$
where $\hat{\mathbf{L}}^{ij}$ corresponds to the graph Laplacian with edge $ij$ removed and $\hat{\mathbf{C}}^{ij}=\left(\beta \mathbf{I} + \hat{\mathbf{L}}^{ij}\right)^{-1}$ is the cavity covariance, i.e. the covariance of $p(\mathbf{Y}|\mathbf{A})$ if we remove edge $i,j$. Now we note that
$$
\phi_{:, ij}^{\top} \hat{\mathbf{C}}^{ij} \phi_{:, ij}=\hat{c}^{ij}_{i,i}+\hat{c}^{ij}_{j,j}-2 \hat{c}^{ij}_{ij}= \kappa_{ij}.
$$
where $\kappa_{ij}$ is the squared distance between the $i^{\text{th}}$ and $j^{\text{th}}$ point under the Gaussian governed by the cavity covariance.

\subsection{Calculating the Trace}
The trace term can be calculated as,
$$
-\frac{1}{2} \operatorname{Tr}\left(\mathbf{Y}\mathbf{Y}^{\top}(\beta \mathbf{I}+\mathbf{L})\right)
$$
$$=-\frac{1}{2} \beta \operatorname{Tr}\left(\mathbf{Y}\mathbf{Y}^{\top}\right)-\frac{\rho}{2} \operatorname{Tr}\left(\mathbf{Y}\mathbf{Y}^{\top} \boldsymbol{\Phi} \mathbf{A} \mathbf{\Phi}^{\top}\right)
$$
$$=-\frac{1}{2} \beta \operatorname{Tr}\left(\mathbf{Y}\mathbf{Y}^\top\right)-\frac{1}{2} \operatorname{Tr}\left(\mathbf{Y}\mathbf{Y}^{\top} \hat{\mathbf{L}}^{i j}\right)- \A_{ij} \rho d_{ij}
$$
where
$$
d_{ij}=\left(\mathbf{y}_{i, :}^{\top} \mathbf{y}_{i, :}-2 \mathbf{y}_{i, :}^{\top} \mathbf{y}_{j,:}+\mathbf{y}_{j,:}^\top \mathbf{y}_{j,:}\right).
$$

\subsection{Mean-field EM Perspective}

In this section, we describe an EM update step implied by assuming the model in \cref{eqn:em-like}.

For approximate EM, we assume a mean field approximation for our distribution $q(\A) = \prod_{i=1}^n \prod_{j<i}^n q(\A_{ij})$. By substituting the forms of the determinant and trace computed above (that split up the terms into terms that involve $\A_{ij}$ and terms that don't), we arrive at the equation below (ignoring terms that don't depend on $\A_{ij}$),
\begin{equation}
q(\A_{ij}) \propto \exp[ \mathbb E_{\mathbf{A_{-ij}}}(\log (1+2\A_{ij}\rho \kappa _{ij})^\frac{d}{2})
- \A_{ij}\rho d_{ij} + \A_{ij} \log \pi_{ij}^{(k-1)} -\A_{ij}\log (1-\pi_{ij}^{(k-1)})]. \; \footnote{Note we fix $\pi_{ij}^{(k-1)}$ and take the 1st variation of the ELBO wrt to $q$ to arrive at this result.}
\end{equation}

As $d \rightarrow \infty$, and setting $\rho = 1/d$,
$$ \log (1+2\A_{ij}\rho \kappa _{ij})^\frac{d}{2} \longrightarrow \log \exp \A_{ij} \kappa_{ij}, $$

and so the variational probabilities are approximately proportional to,
$$
q(\A_{ij}) \propto \exp[ \A_{ij} \mathbb E_{\mathbf{A_{-ij}}}(\kappa_{ij}) - \A_{ij}\rho d_{ij} + \A_{ij} \log \pi_{ij}^{(k-1)} -\A_{ij}\log (1-\pi_{ij}^{(k-1)})],
$$
$$
q(\A_{ij}) \propto \exp \left[ \A_{ij} \left[\mathbb E_{\mathbf{A_{-ij}}}(\kappa_{ij}) - \dfrac{d_{ij}}{d} + \log \dfrac{\pi_{ij}^{(k-1)}}{1-\pi_{ij}^{(k-1)}} \right] \right].
$$

Because $\A_{ij}$ can only be zero or one (and hence have a Bernoulli distribution), we obtain,
\begin{align*}
    q(\A_{ij}=1) &= \frac{\exp \left[\mathbb E_{\mathbf{A_{-ij}}}(\kappa_{ij}) - \dfrac{d_{ij}}{d} + \log \dfrac{\pi_{ij}^{(k-1)}}{1-\pi_{ij}^{(k-1)}} \right]}{ 1 + \exp \left[\mathbb E_{\mathbf{A_{-ij}}}(\kappa_{ij}) - \dfrac{d_{ij}}{d} + \log \dfrac{\pi_{ij}^{(k-1)}}{1-\pi_{ij}^{(k-1)}} \right]} \\
    &= \sigma\left(\left[\mathbb E_{\mathbf{A_{-ij}}}(\kappa_{ij}) - \dfrac{d_{ij}}{d} + \sigma^{-1}( \pi_{ij}^{(k-1)}) \right]\right).
\end{align*}
Using these results as part of the coordinate ascent variational inference (CAVI, \cite{blei-vi}), results in an expectation (E) update step,
$$ \forall i \neq j: q^{k+1}(\A_{ij}=1) = \pi_{ij}^{(k+1)} = \sigma\left(\left[\mathbb E_{\mathbf{A_{-ij}}}(\kappa_{ij}) - \dfrac{d_{ij}}{d} + \sigma^{-1}( \pi^k_{ij}) \right]\right). $$
Note that in the MFVI setting $q^{k+1}(\A_{ij}=1) = \pi_{ij}^{(k+1)}$ as this brings the second term in Equation \ref{eqn:probdr-obj} to 0. We believe that the term $\mathbb E_{\mathbf{A_{-ij}}}(\kappa_{ij})$ approximates $d_{ij}/d$ for large $d$ due to \cref{thm:norm_dist}. We hope that this framework allows for future research with the specified model methodology.

\section{Alternative generative models}
\label{app:alt-gen-models}

Here, we describe other potential generative models that can be specified given an adjacency matrix $\A$.

\subsection{Gaussian Bayesian Networks}
\label{bayesnet}
\cite{pml} describes the joint distribution of a Gaussian directed acyclical graphical model, providing a generative model for our framework. It appears as,
\begin{align*}
    \Y &\sim \mathcal{MN}(0, \M\M^T, \I),
\end{align*}
where $\M$ is a lower triangular matrix (the Cholesky decomposition of the covariance) such that $\M = (\I - \A)^{-1}$ and $\A$ is a row-normalised lower triangular adjacency matrix. This generative model is equivalent to:
$$ \Y_{ij} | \text{pa}(i) \sim \mathcal N \left( \frac{1}{|\text{pa}(i)|}\sum_{k \in \text{pa}(i)} \Y_{kj}, 1 \right), $$
where $\text{pa}(i)$ is the set of points that are parents to point $i$, and $|.|$ denotes the size of a set.

\subsection{Graph Convolutional Gaussian Processes}
\label{section:decoders-conv}

The graph convolutional Gaussian process (GCGP), described in \cite{gcgp, gcgp0}, is defined as
\begin{equation*}
    \Y \sim \mathcal{MN}(0, \mathbf{S}^k \C [\mathbf{S}^k]^T, \I),
\end{equation*}
where $\C$ is a kernel matrix, $\mathbf{S}^k$ is a normalized adjacency matrix defined using $\Tilde{\A} = \A + \I$, raised to the $k$-th power. Taking $\C$ to be an identity matrix provides a potential generative model for our framework.

Note that due to the Neumann expansion, the Cholesky decomposition of the covariance in \cref{bayesnet} can be written as,
$$\mathbf{M} = (\I - \tilde{\A}^T)^{-1} = \sum_{k=0}^\infty (\tilde{\A}^{T})^{k}.$$
This shows that a sum of GCGPs (using increasing powers of the adjacency matrix, without self edges) approximates the covariance in \cref{bayesnet} when the graph is a DAG. The expansion also shows that the covariance of \cref{bayesnet} is composed of all possible hops in the graph, as powers of an adjacency matrix have the interpretation of storing the number of paths from each node to another \citep{barber}.


\section{Connections to other work}

The proposed framework could be changed by directly specifying a generative model on data $\Y$ conditioning on latent variables $\X$ and a further set of latent variables $\X'$ that don't affect $\A'$ directly. This is illustrated in \cref{graph-extension}. Motivations to do this include improving the separation of clusters of points that belong to different labels (e.g. cell types) in recovered embeddings by a GPLVM or a VAE, which (t-)SNE and UMAP can be more performant at.

\begin{figure}[h]
\centering
\begin{tikzpicture}
\node[latent] (A) {$\A'$};
\node[latent, above=of A] (X) {$\X$};
\node[obs, left=of A] (Y) {$\Y$};
\node[latent, above=of Y] (Xp) {$\X'$};

\edge{X}{A};
\edge[dashed]{A}{Y};
\edge[dashed]{X}{Y};
\edge[dashed]{Xp}{Y};

\node[above] at (current bounding box.north) {Generative Model};
\end{tikzpicture}
\qquad
\begin{tikzpicture}
\node[obs] (A) {$\A'$};
\node[latent, above=of A] (X) {$\X$};
\node[obs, right=of X] (Y) {$\Y$};

\edge{Y}{A};
\edge[dashed]{Y}{X};

\node[above] at (current bounding box.north) {Variational Approximation};
\end{tikzpicture}
\caption{Class of possible extensions of the framework, where latent variables $\X$ and a further set of latent variables $\X'$ that don't affect $\A'$ may be used to describe the distribution of $\Y$ directly. Dashed edges show connections that may or may not be added and greyed nodes show observed random variables or parameters.}
\label{graph-extension}
\end{figure}
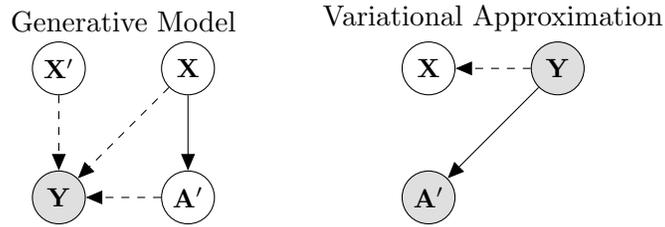

In such models, the ELBO is given by,
\begin{align*}
\mathcal L = \; & \mathbb{E}_{q(\A'|\Y)q(\X)q(\X')}[\log p(\Y|\A', \X', \X)] \\
-& \text{KL}(q(\A'|\Y)||p(\A' | \X)) \\
-& \text{KL}(q(\X|\Y)||p(\X)) \\
-& \text{KL}(q(\X')||p(\X')).
\end{align*}
Such objectives, where a t-SNE/UMAP style loss is added to that of another model (e.g. scvis \citep{scvis} and GPLVMs with t-SNE objectives \citep{gplvm-tsne}), and models that use neural networks for amortised inference of $\X$, appear frequently in literature.

\end{document}